\documentclass[10pt]{amsart}

\usepackage{graphicx} %
\usepackage{amsfonts, amsmath, amsthm, amssymb}
\usepackage{mathtools}
\usepackage[dvipsnames]{xcolor}
\usepackage[
    pdftex=true,
    colorlinks=true,
    linkcolor=RoyalPurple,   %
    citecolor=Blue,     %
    urlcolor=Violet     %
]{hyperref}
\usepackage{float, comment}
\usepackage{algorithm}
\usepackage{algorithmic}
\usepackage{cleveref}
\usepackage[utf8]{inputenc}
\usepackage[T1]{fontenc}
\usepackage{empheq}
\usepackage{amsfonts}
\usepackage{tikz-cd}
\usepackage{caption}
\usepackage{subcaption}
\usepackage{url}
\usepackage{enumitem}

\def\T{{\top}}

\def\phi{{\varphi}}

\def\R{{\mathbb R}}

\def\cal{\mathcal}

\def\cL{{\cal L}}

\def\cP{{\cal P}}

\def\rank{\operatorname{rank}}

\def\a{\alpha}

\def\e{\varepsilon}

\def\s{\sigma}

\usepackage[backend=bibtex, style=alphabetic]{biblatex}
\theoremstyle{plain}
\newtheorem{theorem}{Theorem}[section]
\newtheorem{lemma}[theorem]{Lemma}

\newtheorem*{theorem*}{Theorem}
\newtheorem*{prop*}{Proposition}

\theoremstyle{definition}

\newtheorem{assumption}[theorem]{Assumption}

\title[]{The Tree-SNE Tree exists}
\author[]{Jack Kendrick}

\begin{document}

\begin{abstract} The clustering and visualisation of high-dimensional data is a ubiquitous task in modern data science. Popular techniques include nonlinear dimensionality reduction methods like t-SNE or UMAP. These methods face the `scale-problem' of clustering: when dealing with the MNIST dataset, do we want to distinguish different digits or do we want to distinguish different ways of writing the digits? The answer is task dependent and depends on scale. We revisit an idea of Robinson \& Pierce-Hoffman that exploits an underlying scaling symmetry in t-SNE to replace 2-dimensional with (2+1)-dimensional embeddings where the additional parameter accounts for scale. This gives rise to the t-SNE tree (short: tree-SNE). We prove that the optimal embedding depends continuously on the scaling parameter for all initial conditions outside a set of measure 0: the tree-SNE tree exists. This idea conceivably extends to other attraction-repulsion methods and is illustrated on several examples.
\end{abstract}

\maketitle

\section{Introduction and Results}\label{sec:intro}

\subsection{The problem of clustering.}
From fields such as computer vision to computational linguistics, large sets of high dimensional data are increasingly commonplace. Thus, it is necessary to develop tools that can be used for their analysis and visualisation. While traditional linear techniques are often remarkably powerful, the visualisation problem poses a new type of challenge. Even if data in 1000 dimensions is intrinsically close to a 5-dimensional submanifold (the manifold hypothesis), embedding five dimensions into two or even three dimensions poses a considerable challenge. One popular dimensionality reduction technique is t-distributed stochastic neighbour embedding (t-SNE), first introduced by van der Maarten and Hinton in \cite{vandermaarten2008tsne}. As a completely non-linear method, t-SNE does not try to produce an isometric or low-distortion embedding. Instead, it tries to group the data into clusters. This works remarkably well in practice and has been very widely used. Indeed, it has become apparent \cite{bohm2022attraction} that a broad range of methods (including Laplacian eigenmaps \cite{belkin2003laplacian}, ForceAtlas2 \cite{jacomy2014forceatlas2}, UMAP \cite{mcinnes2018umap, kobak2019umap} and t-SNE \cite{vandermaarten2008tsne}) can be seen as special cases of a general attraction-repulsion method: we start with points in $\mathbb{R}^2$ that represent the high-dimensional data and then induce each point to move towards other points in $\mathbb{R}^2$ that represent `similar' high-dimensional data. Simultaneously, points are repelled by all of the other data points. If the underlying data is highly clustered, then the same cluster structure in $\mathbb{R}^2$ is a steady state of that dynamical system, see \cite{linderman2017tsne, arora2018analysis, cai2022theoretical}.

\subsection{The problem of scale} Since t-SNE embeddings exhibit cluster structures that are representative of those seen in the high-dimensional data, a typical use case for t-SNE is producing a 2D visualisation that aids with identification of clusters within a dataset. However, the question of scale is often overlooked in this clustering process. Although often treated as a static problem, clustering may be viewed as a dynamic problem with a hierarchical structure. As a motivating example, consider the MNIST dataset which contains 70,000 images of handwritten numerical digits. Typically, these images are clustered based on what digits they represent. However, this clustering is not helpful if we are, for example, trying to distinguish different handwriting styles within the dataset. In this case, a more granular clustering is needed. Since many digits can be written in multiple ways, we expect there to be subclusters corresponding to different handwriting styles within the clusters corresponding to each digit. \Cref{fig:mnist-schematic} gives a schematic of example clusters and subclusters within the MNIST dataset.

\begin{figure}[H]
    \includegraphics[width=0.9\textwidth]{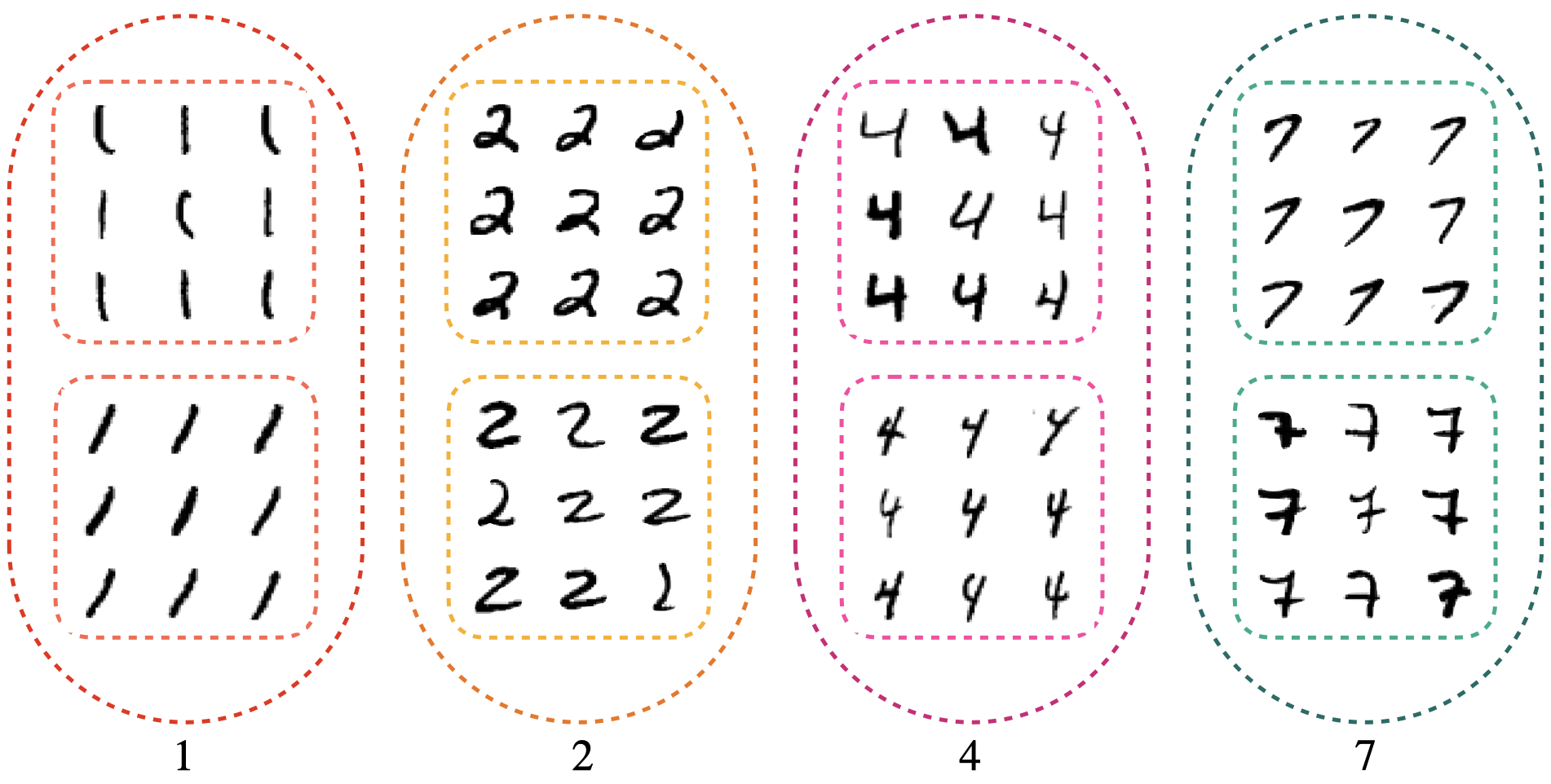}
    \caption{A schematic of subclusters in the MNIST dataset.}
    \label{fig:mnist-schematic}
\end{figure}
One solution to the problem of scale are {\em hierarchical clustering} methods, see \cite{hierarchicalclustering,hierarchicalclustering2} for an overview. These methods are often computationally expensive and so an alternative method is desirable. It has been observed that variants of t-SNE using kernels with heavier tails tend to produce clusters with finer granularity, see for example \cite{kobak2020heavytailed}. However, there is currently no good way to choose a kernel that corresponds to a particular granularity outside of trial and error. Thus, the problem of scale persists within t-SNE embeddings and clusterings. Note that this is not exclusive to t-SNE and is a problem with {\em all} clustering methods yet is often ignored in practice. 

\subsection{Tree-SNE}
To address the problem of scale, we explore a hierarchical variant of t-SNE known as tree-SNE that was introduced in \cite{robinson2020tree-sne}. The idea is to exploit a one-parameter family \cite{kobak2020heavytailed} of kernels that produce t-SNE-type embeddings with finer and finer granularity. We begin with generating an ordinary t-SNE embedding of the data. Then, additional layers are generated by changing the parameter in the kernel and optimizing a certain loss function. Each embedding forms a layer of the tree-SNE embedding and the points in each layer can be interpolated to plot smooth trajectories of each of the data points. \Cref{fig:slicing} contains slices of a tree-SNE embedding of the MNIST dataset, along with the tree that comes from interpolating each layer.

We see in \Cref{fig:slicing} that each of the main branches of the tree splits into smaller branches as the layers increase. These correspond to different ways of writing the same digit. In \Cref{fig:mnist-clusters}, we show example branches and subbranches that correspond to the clusters and subclusters in \Cref{fig:mnist-schematic}. Each branch is labelled with the average image of elements belonging to the cluster.

\begin{figure}[h]
    \includegraphics[width=0.85\textwidth]{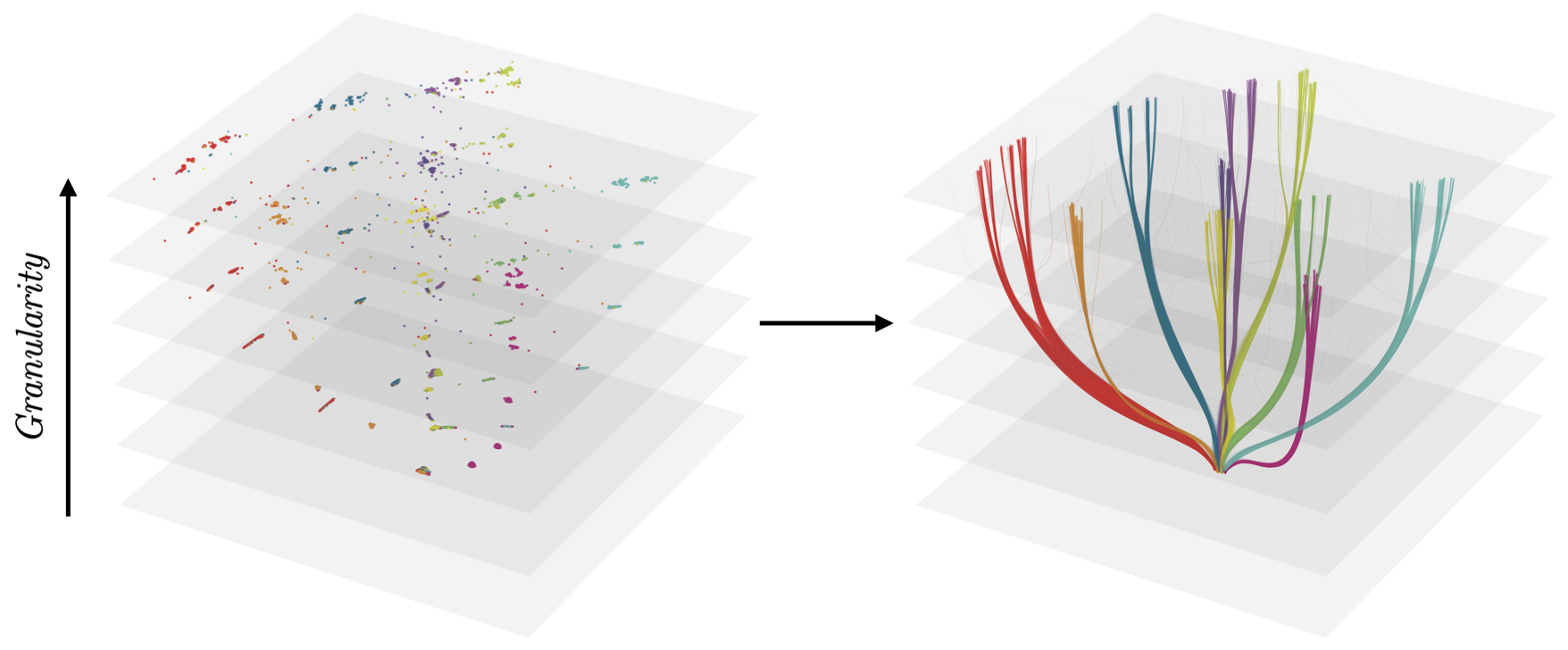}
    \caption{Slices of the MNIST tree-SNE embedding are interpolated to create a tree. Each color corresponds to a digit.}
    \label{fig:slicing}
\end{figure}

In particular, we see that as the layers of the embedding increase there are clusters that reveal structure with finer granularity and there is a continuous structure to the tree: the points follow a continuous trajectory and do not jump around from layer to layer. As has been previously observed in \cite{kobak2020heavytailed}, t-SNE embeddings generated with heavier tailed kernels tend to produce finer clusters and so it is expected that clusters at the top of the tree-SNE embedding are more granular than those at the bottom. However, the continuous structure of the tree is non-trivial and will be the main focus of this paper.

\begin{figure}[h]
    \includegraphics[width=0.6\textwidth]{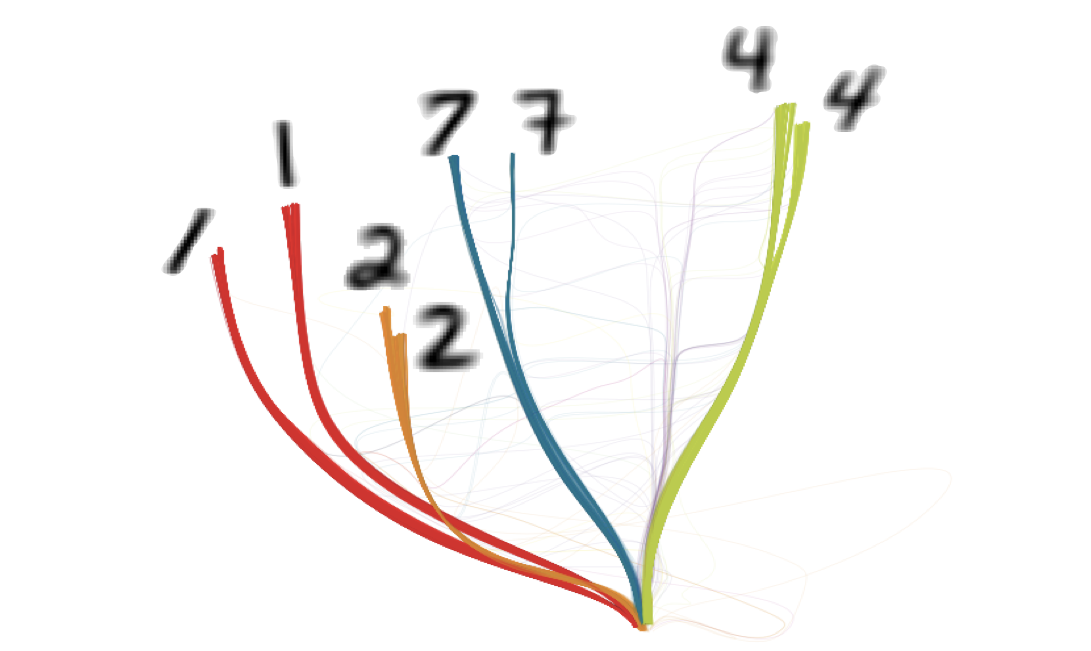}
    \caption{Subbranches of the MNIST tree correspond to different handwriting styles. Images are the average of elements in the corresponding cluster.}
    \label{fig:mnist-clusters}
\end{figure}

\subsection{Main Result}

Our main result concerns the continuous structure of tree-SNE embeddings, as observed in \Cref{fig:slicing}. We prove that this structure is not coincidental and is a consequence of the tree-SNE algorithm's design. Recall that each layer of a tree-SNE embedding is generated using a kernel from a one parameter family. We show that as the parameter varies continuously, so does the embedding.

\begin{theorem*}[The Tree-SNE Tree exists]\label{thm:intro}
    Suppose $(y_1,\ldots, y_n)\in\R^{nd}$ is a \emph{t-SNE} embedding in $\R^d$ with parameter $\a,$ Under mild conditions, for each $\a'$ in some neighbourhood of $\a$ there is a \emph{t-SNE} embedding $(y_1', \ldots, y_n')\in\R^{nd}$ in a neighbourhood of $(y_1, \ldots, y_n).$
\end{theorem*}

To be more specific, we view the point $(\a, y_1, \ldots, y_n)\in\R^{1+nd}$ as belonging to the zero set of the function $F.$ The function $F$ is the gradient of the loss function minimized to find t-SNE embeddings when the parameter $\a$ is not treated as a constant. The `mild conditions' referred to in our theorem ensure that $F$ is smooth and generically has a constant rank. This will allow us to conclude that the zero set of $F$ is generically a smooth manifold of a dimension $1+ d(d+1)/2$ around $(\a, y_1, \ldots, y_n),$ with one of these dimensions accounting for changes in the parameter $\a.$ These `mild conditions' are necessary but are not restrictive. Indeed, they are satisfied by generic datasets and initializations. We provide a detailed discussion of these conditions in \Cref{sec:proof}.

\subsection*{Organisation.}

In \Cref{sec:prelims}, we recount the necessary background information on t-SNE and describe the tree-SNE algorithm. In \Cref{sec:proof}, we prove our main theorem to show that the tree-SNE algorithm indeed produces tree-like embeddings for generic datasets and generic initialisations. Finally, in \Cref{sec:examples} we use tree-SNE to visualise two different datasets from the HathiTrust library and examine the structures revealed in the embeddings.

\section{Preliminaries}\label{sec:prelims}

In this section, we give a brief overview of both t-SNE and tree-SNE. Throughout this section, we fix a dataset $\{x_1, \ldots, x_n\}\subset\R^D,$ where on may think of the dimension $D$ as large. We begin by describing dimensionality reduction and data visualisation with t-SNE.

\subsection{T-SNE}
Given a dataset $\{x_1, \ldots, x_n\}\subset\R^D,$ the central goal of t-SNE is to find a low-dimensional set of points $\{y_1, \ldots, y_n\}\subset\R^d$ that preserves the overall structure of the data. In particular, if $x_i$ and $x_j$ are ``close'' in $\R^D,$ then $y_i$ and $y_j$ should be ``close'' in $\R^d$ with high probability. Typically, the dimension $d$ is chosen such that $d=2$ or $d=3$ and so the t-SNE embedding $\{y_1, \ldots, y_n\}$ can be used to visualise the high-dimensonal dataset.
The t-SNE method consists of two stages. First, {\em affinities} $p_{ij}$ are determined between each pair of points $x_i$ and $x_j.$ The affinity $p_{ij}$ between the points $x_i$ and $x_j$ is a probability that represents the similarity of $x_i$ and $x_j$ within the dataset: a high affinity is associated to pairs that are close together, whereas low affinities are assigned to pairs of points that are far apart. The affinity $p_{ij}$ is determined using {\em directional affinities} $p_{j|i}$ and $p_{i|j}.$ For $j\neq i,$ the directional affinity $p_{j|i}$ is defined as
\begin{align*}
    p_{j|i} &= \frac{\exp(-\|x_i-x_j\|^2/\s_i^2)}{\sum_{k\neq i}\exp(-\|x_i-x_k\|^2/\s_i^2)},
\end{align*} 
where the variance $\s_i^2$ is a user determined parameter. The directional affinity between $x_i$ and itself is set to $p_{i|i}=0$ for all $i.$ Note that for each value of $i,$ the directional affinities determine a probability distribution on $\{x_1, \ldots, x_n\}.$ The variance $\s_i^2$ is chosen such that the {\em perplexity} of this distribution, defined
\begin{equation*}
    \cP = \exp\left(-\log 2 \cdot\sum_{j\neq i}p_{j|i}\log_2 p_{j|i}\right),
\end{equation*}
has some pre-specified value. Typical values of the perplexity $\cP$ are between 5 and 50. For each pair $i, j$ the affinity $p_{ij}$ is the normalised average of the directional affinities $p_{j|i}$ and $p_{i|j},$
\begin{align*}
    p_{ij} = \frac{p_{j|i}+p_{i|j}}{2n}.
\end{align*}

Once affinities between all pairs of points $x_i$ and $x_j$ have been determined, a low dimensional embedding of the data $\{y_1, \ldots, y_n\}\subset\R^d$ is found. For each pair of points $y_i, y_j\in\R^d$ the affinity $q_{ij}$ is defined
\begin{align*}
    q_{ij} &= \frac{K(\|y_i-y_j\|)}{\sum_{k\neq \ell}K(\|y_k-y_\ell\|)}
\end{align*}
where $K$ is a kernel. Traditionally, the kernel $K$ corresponds to the t-distribution with one degree of freedom, $K(d) = (1+d^2)^{-1}.$ However, more general t-distributions can be considered. As in \cite{kobak2020heavytailed}, we consider kernels given by 
\begin{align*}
    K_\a(d) &= \frac{1}{(1+d^{2/\a})^\a}
\end{align*}
where $\a>0$ is a parameter. This corresponds to a scaled t-distribution with $2\a-1$ degrees of freedom. Note that choosing $\a=1$ yields the standard t-distribution kernel and the taking the limit as $\a\to\infty$ gives the Gaussian kernel $K_\infty(d) = \exp(-d^2).$ This corresponds to the SNE method of Hinton and Roweis \cite{hinton2002sne}. Moreover, $\a < 0.5$ corresponds to a t-distribution with negative degrees of freedom and so represents a distribution with heavier tails than any (non-scaled) t-distribution. Note that when $\a<0.5$ the resulting distribution is no longer a probability distribution. A common misconception about the t-SNE method is that the kernel must correspond to a probability distribution, however this is not correct. Since we only consider a finite number of points, distributions may always be renormalized and it is not, for instance, necessary for any particular integral to be finite.
As in the high-dimensional space, the affinity between $y_i$ and itself is set as $q_{ii}=0.$ The embedding $\{y_1, \ldots, y_n\}$ is chosen so that the Kullbeck-Leibler (KL) divergence
\begin{equation*}
    \cL(y_1, \ldots, y_n) = \sum_{i\neq j} p_{ij}\log\frac{p_{ij}}{q_{ij}}
\end{equation*}
is minimized. The gradient of $\cL$ is given by
\begin{equation*}
    \frac{\partial\cL}{\partial y_i} = \sum_{j=1}^n (p_{ij}-q_{ij})K_\a(\|y_i-y_j\|)^{1/\a}(y_i-y_j) 
\end{equation*}
and so an embedding that minimizes the loss function may be found using a random initialisation and applying first order methods such as gradient descent.

\subsection{Tree-SNE} \label{sec:prelims-tree}
We now describe how tree-SNE builds on t-SNE to generate collections of embeddings that reveal finer substructures within data. The central idea of tree-SNE is that iteratively decreasing the degrees of freedom in the t-distribution kernel yields embeddings that capture the structure of data in increasing granularity. This sequence of embeddings can then be stacked atop one another and interpolated to reveal a tree-like structure with different branches corresponding to different subclusters in the data.
Tree-SNE takes in a sequence of parameters $\{a^{(i)}\}_{i=1}^m$ and produces a sequence of embeddings $\{(y_1^{(i)}, \ldots, y_n^{(i)})\}_{i=1}^m.$ As with t-SNE, each embedding is produced in two stages. First, affinities between each pair $x_i, x_j$ of data points in $\R^D$ are calculated so that the directional affinities at each point have some perplexity $\cP^{(i)}$ that is dependent on $\a^{(i)}.$ Then, the embedding $(y_1^{(i)}, \ldots, y_n^{(i)})$ is a minimizer of the loss function $\cL.$ Each embedding $(y_1^{(i)}, \ldots, y_n^{(n)})$ forms a layer of the tree-SNE embedding. The parameters $\{\a^{(i)}\}_{i=1}^m$ are chosen such that $\a^{(1)}=1$ and $\a^{(i+1)}<\a^{(i)}$ for all $i.$ Similarly, the perplexities $\cP^{(i)}$ decrease as the layer of the embedding increases. Thus, as the layers of the tree-SNE embedding increase, the t-distribution has heavier and heavier tails and this leads to finer structures of the data being revealed. 

Typically, t-SNE embeddings are found using a random initialisation. However, since the loss function $\cL$ is highly non-convex, different random initialisations lead to different embeddings. To ensure that the tree-SNE embedding has a continuous structure from one layer to the next, we use the embedding $\{y_1^{(i)}, \ldots, y_n^{(i)}\}$ as the initialisation for layer $i+1.$ Assuming that the difference between $\a^{(i)}$ and $\a^{(i+1)}$ is small, we expect that if $\{y_1^{(i)}, \ldots, y_n^{(i)}\}$ minimizes $\cL$ with parameter $\a^{(i)}$ then there should be a minimizer for $\cL$ with parameter $\a^{(i+1)}$ in some small neighbourhood of $\{y_1^{(i)}, \ldots, y_n^{(i)}\}.$ We prove that this is the case in \Cref{sec:proof}.
Note that the random initialisation of the t-SNE embedding in the first layer may serve as an intrinsic sanity check: since different initialisations produce different trees, any structures that persist across multiple initialisations is likely to be trustworthy. This idea is used in the work \cite{greengard2020factor}.

\section{Proof of Continuous Structure}\label{sec:proof}

In this section, we prove that, under mild assumptions, the tree-SNE method described in \Cref{sec:prelims-tree} produces a sequence of embeddings with a continuous structure. Although tree-SNE is a discrete process with a finite set of parameters producing a finite set of embeddings, we consider a continuous version of the problem and treat the embedding $y_1(\a), \ldots, y_n(\a)$ as a function of the parameter $\a$ which may take any value in $(0, 1].$ Similarly, the perplexity $\cP(\a)$ is also a function of the parameter $\a.$ Our goal is to prove that each function $y_i(\a)$ is continuous. 

We consider the map $F:\R^{1+nd}\to\R^{nd}$ defined by
\begin{equation*}
    F(\a, y_1, \ldots, y_n) = \left(\frac{\partial\cL}{\partial y_1}, \ldots, \frac{\partial\cL}{\partial y_n}\right).
\end{equation*}

Since the embedding $y_1(\a), \ldots, y_n(\a)$ is a minimizer of $\cL,$ in particular the condition $\nabla\cL = 0$ is satisfied and so every embedding lies on the zero-set of $F.$ As the function $\cL$ is not convex, it is not true that every point in the zero-set of $F$ corresponds to a minimizer of $\cL.$ However, given an appropriate choice of parameters such as step size, first order methods like gradient descent return the point in the zero-set of $F$ closest to their initialisation. We restate a refined version of our main result in the following theorem.

\begin{theorem*}[The Tree-SNE tree exists - refined]\label{thm:main}
    Suppose $(\a, y_1, \ldots, y_n)$ is a generic point in the zero set of $F.$ Then, the zero set of $F$ is a smooth, properly embedded, $1+d(d+1)/2$ dimensional submanifold of $\R^{1+nd}$  $(\a, y_1, \ldots, y_n).$ Moreover, in this neighbourhood the parameter $\a$ is not constant.
\end{theorem*}

We argue that this is sufficient to prove the contuous structure of a tree-SNE embedding. Indeed, fix a parameter $\a$ and suppose that $(y_1, \ldots, y_n)$ is the corresponding layer of the tree-SNE embedding. The result of \Cref{thm:main} tells us that for each $\a'$ in some neighbourhood of $\a$ there exists a set of points $(y_1', \ldots, y_n')$ in a neighbourhood of $(y_1, \ldots, y_n)$ so that $(\a', y_1', \ldots, y_n')$ is in the zero set of $F.$ So, by initializing the $\a'$ layer of the tree-SNE embedding at $(y_1, \ldots, y_n),$ gradient descent will return the embedding $(y_1', \ldots, y_n'),$ or some other set of points in the neighbourhood of $(y_1, \ldots, y_n).$ It follows that, when the parameters $\{a^{(i)}\}_{i=1}^m$ are chosen appropriately and each layer is initialized using the preceding layer, the tree-SNE embedding will have a continuous structure. The main tool that we use is the {\em constant-rank level-set theorem}. This is a standard result in the theory of manifolds. For details, see for example \cite{lee2013ism}.

\begin{theorem}[Constant-Rank Level Set]
    Let $M$ and $N$ be smooth manifolds, and $\phi:M\to N$ a smooth map with constant rank $r$ in some neighbourhood of $p.$ Each level set of $\phi$ is locally a properly embedded submanifold of codimension $r$ in $M.$
\end{theorem}

To prove \Cref{thm:main}, we will show that $F$ is smooth and generically has rank $nd-d(d+1)/2.$ Then, using the constant rank level-set theorem, we show that at a generic point, the zero set of $F$ is locally a properly embedded submanifold of $\R^{1+nd}$ of dimension $1+d(d+1)/2$ where $\a$ is not constant. First, we fix the following set of assumptions.

\begin{assumption}\label{ass:setting} 
    We assume the following conditions are satisfied.
    \begin{enumerate}[label=(\Alph*)]
        \item \label{ass:data} The dataset $\{x_1, \ldots, x_n\}\subset\R^D$ is fixed.
        \item \label{ass:perp}The perplexity $\cP(\a)$ is a smooth function of $\a$
        \item \label{ass:rank} Given the affinities $p_{ij}$ calculated using perplexity $\cP(\a),$  there is at least one set of points $y_1, \ldots, y_n\in\R^d$ such that $\nabla^2\cL$ has rank $nd- d(d+1)/2.$
    \end{enumerate}
\end{assumption}

Assumptions \ref{ass:data} and \ref{ass:perp} will ensure that the map $F$ is smooth, whereas Assumption \ref{ass:rank} will ensure that $F$ generically has rank $nd-d(d+1)/2.$ This is the key to ensuring that the constant-rank level set theorem applies in our setting.

\subsection{A Discussion of Assumption \ref{ass:rank}} Note that since the loss function $\cL$ is defined using pairwise distances, it is invariant under all rigid transformations of $\R^d.$ It follows that the rank of $\nabla\cL^2$ is at most $nd-d(d+1)/2$ since the space of all rigid transformations is $d(d+1)/2$ dimensional. Since this is the only invariance exhibited by the function, we expect the Hessian $\nabla\cL^2$ to achieve this rank at any generic embedding $y_1, \ldots, y_n.$ While Assumption \ref{ass:rank} is generically satisfied, it is indeed possible for the Hessian $\nabla^2\cL$ to have rank lower than $nd-d(d+1)/2,$ as shown in the following proposition.

\begin{prop*}[Rank Deficient Hessian]
    Let $n=3$ and $d=2.$ Suppose that the set $\{y_1, y_2, y_3\}\subset\R^2$ is the vertex set of an equilateral triangle with side length $r.$ Then, the rank of $\nabla^2\cL$ is at most $nd - d(d+1)/2 - 1.$
\end{prop*}
\begin{proof}
     Recall that the rank of $\nabla^2\cL$ is at most $nd-d(d+1)/2$ since $\cL$ is invariant under all rigid transformations of $\R^d.$ Note that any additional invariance of $\cL$ causes the rank of the Hessian $\nabla^2\cL$ to drop. Since the set $\{y_1, y_2, y_3\}\subset\R^2$ is the vertex set of an equilateral triangle, in particular we have that for each choice of $i\neq j,$ the distance between $y_i$ and $y_j$ is $\|y_i-j_y\| = r.$ Note that each of the probabilities $q_{ij}$ is given by 
    $$q_{ij} = \frac{K(r)}{\sum_{k\neq \ell}K(r)}= \frac{1}{6}$$
    and so is independent of the side length $r.$ It follows that $\cL(y_1, y_2, y_3)$ is invariant under scalings of the side length $r.$ This is not a rigid transformation of $\R^2$ and so the rank of $\nabla^2\cL$ is at most $nd - d(d+1)/2 - 1.$
\end{proof}

Although Assumption \ref{ass:rank} only requires that the desired rank be achieved for one set of points, if this condition is satisfied, the Hessian at any generic embedding $y_1, \ldots, y_n$ will have rank $nd-d(d+1)/2$ due to the lower semicontinuity of rank. This the content of the following lemma.

\begin{lemma}
    Suppose Assumption \ref{ass:rank} is satisfied. Then, for any generic set of points $y_1, \ldots, y_n$ the Hessian $\nabla^2\cL$ has rank $nd-d(d+1)/2.$
\end{lemma}

\begin{proof}
    For any set of points $y_1, \ldots, y_n,$ the rank of $\nabla^2\cL$ is at most $r = nd-d(d+1)/2$ since $\cL$ is invariant under all rigid transformations of $\R^d.$ Suppose the rank is equal to $r$ at $y_1, \ldots, y_n$. Then, there exists an $r\times r$ minor $m(y_1, \ldots, y_n)$ of $\nabla^2\cL$ that is non-zero at $y_1, \ldots, y_n.$ Note that $\cL$ is analytic and so this minor is also analytic. Thus, the zero set of $m(y_1, \ldots, y_n)$ is measure zero. Moreover, the zero set of all $r\times r$ minors of $\nabla^2\cL$ is measure zero and so a generic embedding does not land in this set.
    Since the rank function is lower semicontinuous, $\rank\nabla^2\cL$ may only increase in a neighbourhood of $y_1, \ldots, y_n.$ If the rank of the Hessian $\nabla^2\cL$ is equal to $nd-d(d+1)/2$ at $y_1, \ldots, y_n,$ it follows that $\rank \nabla^2\cL$ is constant in a neighbourhood of $y_1,\ldots, y_n.$ It follows that at a generic set of points, the rank of $\nabla^2\cL$ is exactly $nd-d(d+1)/2.$
\end{proof}

Note that if Assumption \ref{ass:rank} is not satisfied, then it will be satisfied for a random small perturbation of the affinities. This corresponds to a small perturbation of the perplexity of each distribution. Since t-SNE is known to be robust to changes in perplexity \cite{vandermaarten2008tsne}, this does not have a significant effect on the embedding.

\begin{lemma}
    Suppose for a set of affinities $p_{ij}$ that Assumption \ref{ass:rank} is not satisfied. Then, replacing $p_{ij}$ with $p_{ij}+\e_{ij}$ with $\e_{ij}$ sampled i.i.d from $N(0, \s^2),$ satisfies Assumption \ref{ass:rank} with probability 1.
\end{lemma}
\begin{proof}
    Since $p_{ij}$ is not a function of the embedding $y_1, \ldots, y_n,$ each entry of $\nabla^2\cL$ is linear in $p_{ij}.$ Note that $\nabla^2\cL$ has rank strictly less than $nd-d(d+1)/2$ if and only if all of its $nd-d(d+1)/2$ minors are equal to 0. The set of all $p_{ij}$ satisfying these polynomial constraints is measure zero. Thus, any random perturbation of the affinities will not land in this set with probability 1.
\end{proof}

\subsection{Proof of Theorem.} We now show that, in the setting of Assumption \ref{ass:setting}, the map $F:\R^{1+nd}\to\R^{nd}$ is smooth and has constant rank $nd-d(d+1)/2$ in some neighbourhood of a generic point $(\a, y_1, \ldots, y_n)\in\R^{1+nd}.$

\begin{lemma}\label{lem:hypothesis}
    The map $F:\R^{1+nd}\to\R^{nd}$ is smooth and has constant rank $nd-d(d+1)/2$ around a generic point $(\a, y_1, \ldots, y_n).$
\end{lemma}

\begin{proof}
    Recall that the partial derivative of $\cL$ with respect to $y_i$ is given by
\begin{equation*}
    \frac{\partial\cL}{\partial y_i} = 4\sum_{j\neq i}(p_{ij}-q_{ij})K_\a(\|y_i-y_j\|)^{1/\a}(y_i-y_j).
\end{equation*}
Assuming that $y_i\neq y_j$ for all $i\neq j,$ note that $q_{ij}$ and $K_\a(\|y_i-y_j\|)$ are smooth functions of $y_1, \ldots, y_n.$ It follows that $\frac{\partial\cL}{\partial y_i}$ is a smooth function of $y_1, \ldots, y_n.$ We now justify that $\frac{\partial\cL}{\partial y_i}$ is a smooth function of $\a.$ Assuming that $\a\neq 0,$ it is clear that $q_{ij}$ and $K_\a(\|y_i-y_j\|)^{1/\a}$ are smooth functions of $\a.$ Thus, we need only justify that $p_{ij}$ is a smooth function of $\a.$
Recall that the affinities $p_{ij}$ are smooth functions of the bandwidths $\s_i$ chosen to assure that the distribution around each data point has the perplexity $\cP(\a).$ Note that it is sufficient to prove that the bandwidths $\s_i$ are smooth functions of $\a.$
By assumption, the perplexity $\cP(\a)$ is a smooth function of $\a.$ Consider the function $G:\R^{1+n}\to\R^n$ defined by
\begin{equation*}
    G(\a, \s_1, \ldots, \s_n) = \begin{bmatrix}
        \cP(\a) - \exp(-\log 2 \sum_{j\neq i}p_{j|i}\log_2 p_{j|i})
    \end{bmatrix}_{i=1}^n
\end{equation*}
where we view each directional affinity $p_{j|i}$ as a smooth function of the bandwidth $\s_i.$ The set of parameters and bandwidths used to find the tree-SNE embeddings are in the zero set of $G.$ 
It is clear the $G$ is a smooth function. Suppose that for a given $\a^*,$ $\s_1^*,\ldots, \s_n^*,$ the equality $G(\a^*, \s_1^*, \ldots, \s_n^*) = 0$ holds. Then, by the implicit function theorem, $\s_1, \ldots, \s_n$ are a continuous function of $\a$ in a neighbourhood of $(\a^*, \s_1^*, \ldots, \s_n^*)$ if the matrix
\begin{align*}
   J= \begin{bmatrix}
        \frac{\partial}{\partial\s_k}( \cP(\a) - \exp(-\log 2 \sum_{j\neq i}p_{j|i}\log_2 p_{j|i}))
    \end{bmatrix}_{i,k=1}^n
\end{align*}
is invertible. Note that $J$ is diagonal and so is not invertible whenever the equality $$\frac{\partial}{\partial\s_i}( \cP(\a) - \exp(-\log 2 \sum_{j\neq i}p_{j|i}\log_2 p_{j|i}))=0$$ for some value of $i.$ Generically, the equality does not hold and so the implicit function theorem applies. It follows that $F$ is a smooth function as desired. 
We now show that $F$ has constant rank $nd-d(d+1)/2$ in some neighbourhood of a generic point $(\a, y_1, \ldots, y_n).$ Recall that the rank of $F$ is the rank of its Jacobian $\nabla F.$ The Jacobian $\nabla F$ is given by
\begin{align*}
    \nabla F = \begin{bmatrix}
        \frac{\partial^2\cL}{\partial y_1\partial\a} & \ldots & \frac{\partial^2\cL}{\partial y_n\partial\a} \\
        & \nabla^2 \cL &
    \end{bmatrix}
\end{align*}
and so the rank of $F$ is at least the rank of the Hessian $\nabla\cL^2.$ Note that $\nabla F$ is an $(nd+1)\times nd$ matrix and so its rank cannot exceed $nd.$ However, like $\cL,$ the function $F$ is invariant under all rigid transformations of $\R^d.$ The space of all rigid transformations of $\R^d$ is $d(d+1)/2-$dimensional and so it follows that $\rank\nabla F\leq nd - d(d+1)/2.$ By assumption, $\nabla^2\cL$ generically has rank $nd-d(d+1)/2$ and so we must have that $\rank\nabla F = nd-d(d+1)/2.$ By the lower semicontinuity of rank, $F$ has constant rank of $nd-d(d+1)/2$ in a neighbourhood of a generic point, as desired.
\end{proof}

We now apply the constant-rank level-set theorem to prove our main theorem.

\subsubsection*{Proof}

Choose a generic point $(\a, y_1, \ldots, y_n)$ in the zero set of $F.$ Then, as shown in \Cref{lem:hypothesis}, the hypothesis of the constant-rank level-set theorem is satisfied. In particular, we have that the zero-set of $F$ around $(\a, y_1, \ldots, y_n)$ is a properly embedded, $1+d(d+1)/2$ dimensional submanifold of $\R^{1+nd}.$ 
It remains to argue that the parameter $\a$ is not constant in this neighbourhood. Indeed, since $F$ is invariant under all rigid transformations of $\R^d,$ any level set of $F$ is at least $d(d+1)/2$ dimensional. Suppose that $\a$ is constant in this neighbourhood. Then, there exists some nontrivial vector $(v_1, \ldots, v_n)^\T\in \R^{nd}$ such that $(\a, y_1+v_1, \ldots, y_n+v_n)=0.$ However, we would then have that $(v_1, \ldots, v_n)^\T$ is in the nullspace of $\nabla^2\cL.$ Since $\nabla^2\cL$ has rank $nd-d(d+1)/2$ by assumption, this is a contradiction. It follows that $\a$ is not constant, as desired. \qed

\section{Examples}\label{sec:examples}

W use tree-SNE to visualize two high-dimensional datasets constrcuted from the HathiTrust Library. The full HathiTrust dataset consists of 13.6 million works described by millions of features that correspond to word counts on each page of the work. We use the preprocessed dataset from \cite{Schmidt2018Stable} that reduces each work to a 100-dimensional vector.  For each example, we generate a tree-SNE embedding using the default parameters described in \cite{robinson2020tree-sne}. Each layer of the embeddings was found using FIt-SNE \cite{Linderman2019fitsne} with default parameters. Code to reproduce these examples, as well as the MNIST example seen in \Cref{sec:intro}, can be found at \url{https://github.com/j4ck-k/tree-sne}.

\begin{figure}[H]
    \includegraphics[width=0.59\textwidth]{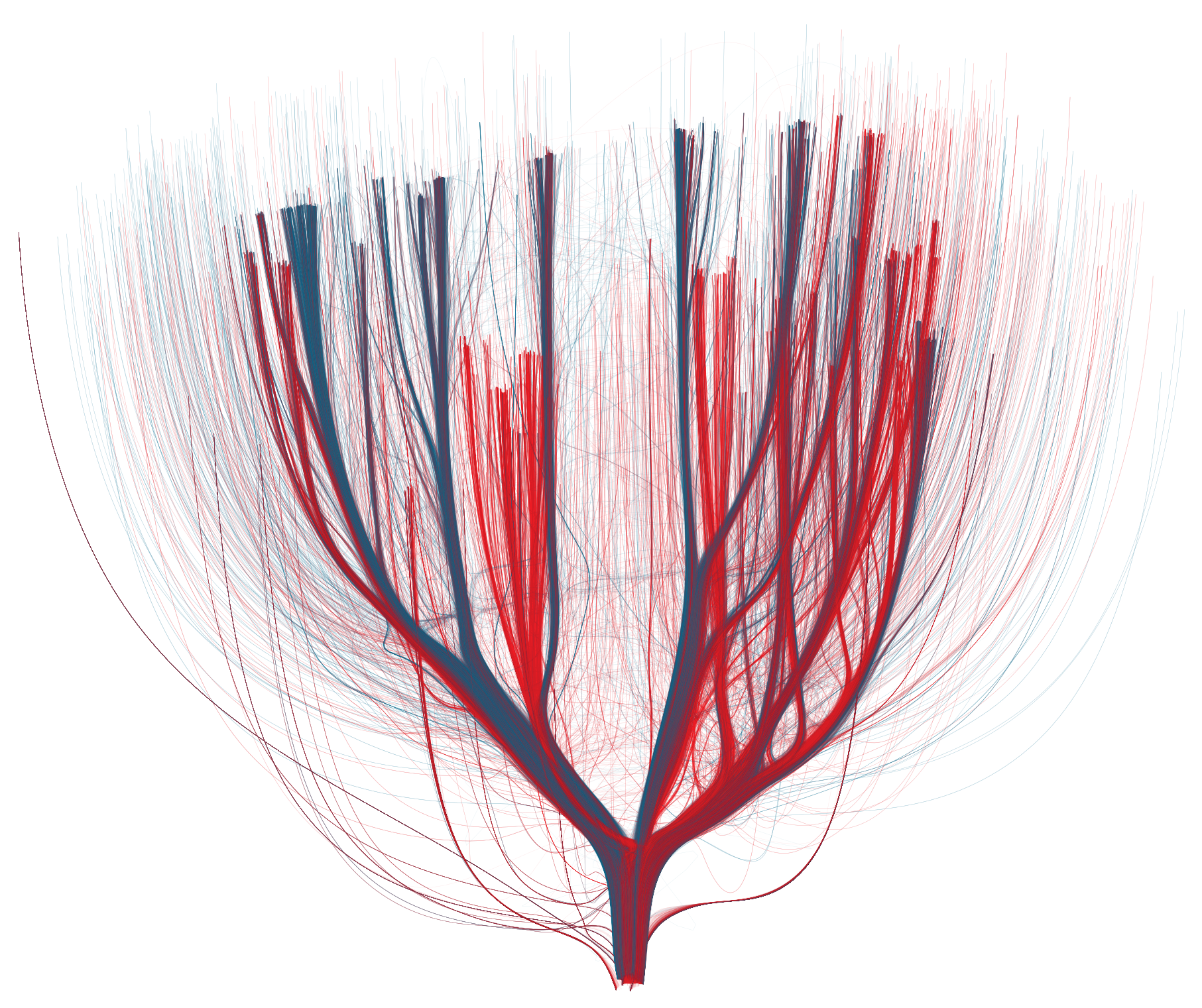}
    \includegraphics[width=0.4\textwidth]{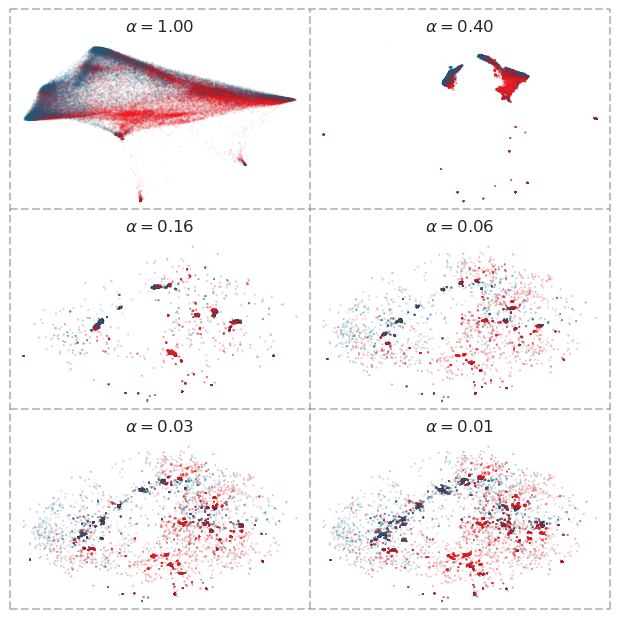}
    \caption{Tree-SNE embedding of 100,000 works of American and English literature. The red points correspond to English literature and blue points correspond to American literature.}
    \label{fig:hathitrust-tree}
\end{figure}

\subsection{English and American Literature}\label{sec:hathitrust}

For our first example, we use tree-SNE to visualize a set of 100,000 works of English and American literature from the HathiTrust library dataset. This dataset was constructed from the full HathiTrust dataset by filetering works classified as either PR (the Library of Congress code for English literature) and PS (the Library of Congress code for American literature) and taking the first 100,000 works for computational simplicity and efficiency. Our dataset is approximately 55\% English and 45\% American literature.

We generates a 50 layer embedding with degrees of freedom between 1 and 0.01. Figure~\ref{fig:hathitrust-tree} shows the tree-SNE embedding of the dataset along with slices of several layers. Points are color coded depending on Library of Congress code. In each plot, red points are works classified as PR and blue points are works classified as PS. When $\a=1$ we see that, although the two classifications are being pushed to opposite sides of the embedding, there is no separation between works of English and American literature. Clusters begin to appear in subsequent layers of the plot.

To illustrate how the tree-SNE embedding captures the evolution of clusters within the data, we track the paths of 109 works that form a single cluster roughly halfway into the tree-SNE embedding. This cluster was chosen because it contains a mix of American (80 works) and English (29 works) literature and its relatively small size allows us to easily examine the works that appear as well as subclusters that form within the data. The trajectories of these works, as well as slices showing the evolution of the cluster throughout the embedding, are shown in \Cref{fig:cluster-trajectory}. 

\begin{figure}[h]
    \includegraphics[width=0.59\textwidth]{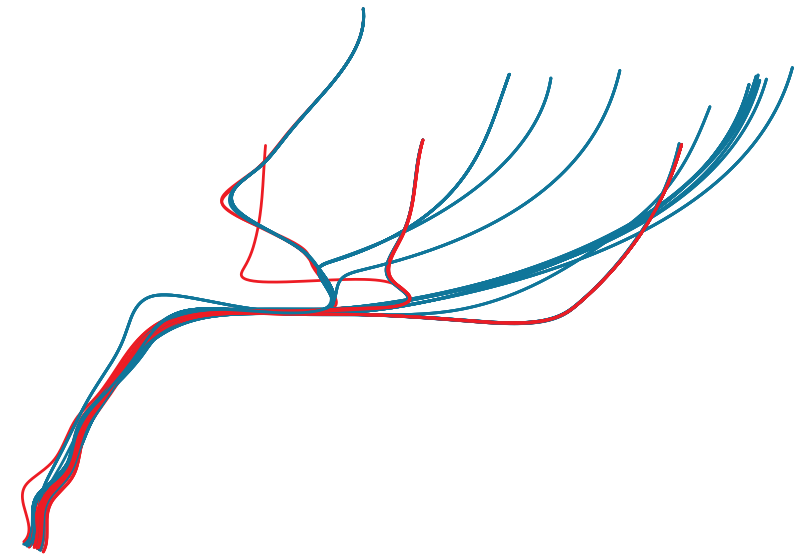}
    \includegraphics[width=0.4\textwidth]{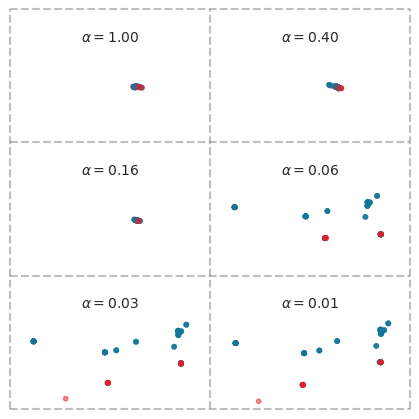}
    \caption{The trajectories of 109 works in the tree-SNE embedding in \Cref{fig:hathitrust-tree}. These works form a cluster in the tree approximately halfway through the embedding and subsequently split into more informative subclusters towards the top of the tree.}
    \label{fig:cluster-trajectory}
\end{figure}

The 109 works in this cluster are mostly pieces of travel writing and journals by writers from the 19th century. Given that these works center a common theme, a macro-level cluster is expected. We see that as the cluster evolves within the tree, additional branches begin to appear. Notably, in the final layer of the embedding, all except one subcluster is formed entirely of either English or American literature. The one exception is a cluster containing 13 works of English literature and 18 works of American literature. However, the American literature in this subcluster consists entirely of travel writings from the novelist Nathaniel Hawthorne during his time in the United Kingdom and continental Europe.

\subsection{The Anti-Stratfordian Theory} Although it is widely accepted that the plays and poems attributed to Shakespeare were indeed written by William Shakespeare of Stratford-upon-Avon, there is a fringe theory amongst literary historians that these works were produced either by another author or a group of writers. This is known as the Anti-Stratfordian Theory. Motivated by the cluster formed by Shakespeare and his contemporaries in \Cref{fig:hathitrust-tree}, we apply tree-SNE to explore the Shakespearean authorship question. 

Two prominent alternative candidates for Shakespearean authorship are the philospher Sir Francis Bacon and the playwright Christopher Marlowe. We constructed a dataset of 2419 English language works by Shakespeare (68 works), Bacon (87 works), Marlowe (114 works), and additional contemporaries from the HathiTrust database. Many of these contemporaries have also been proposed as candidate authors for Shakespeare's work yet are not as popular as Marlowe and Bacon. Since this dataset is relatively small, we generate a tree-SNE embedding with only 10 layers corresponding to degrees of freedom between 1 and 0.25. The embedding and slices of three layers are in \Cref{fig:shakespeare-tree}.

\begin{figure}[h]
    \includegraphics[width=0.7\textwidth]{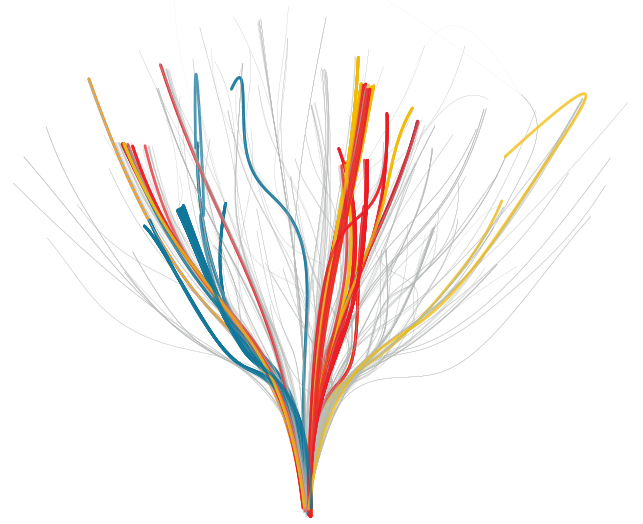}
    \includegraphics[width=0.2\textwidth]{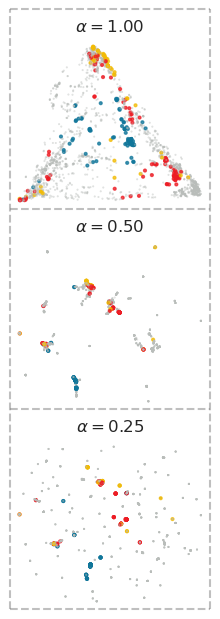}
    \caption{Tree-SNE embedding of Shakespeare and contemporaries. Gold points are works attributed to Shakespeare, red points are works by Marlowe, and blue points are works by Bacon.}
    \label{fig:shakespeare-tree}
\end{figure}

We applied DBSCAN to various layers in the embedding to determine clusters within the data. In the initial t-SNE embedding, there is little separation between any of the included authors, although we see that works by Bacon are more central and works by Marlowe and Shakespeare are pushed towards the corners of the embedding. In particular, the works of Shakespeare are pushed to the north whereas works of Marlowe are scattered throughout. As the layers of the embedding increase, we see that clusters consisting of works attributed to a single author start to appear. In the final layer of the embedding, we see that there are two larger clusters containing works of Shakespeare. More than half of the works attributed to Shakespeare appear in one large cluster of 304 works. This cluster mostly consists of plays and many authors including Marlowe, but not Bacon, are represented. Other than Shakespeare, ten authors (including Marlowe) appear at least ten times in this cluster. Bacon and Shakespeare only appear together in one cluster that contains 6 works of Bacon, 7 works of Shakespeare, and 256 works total. Thus, the vast majority of the works attributed to these authors do not appear in the same clusters within the data. Within the confines of this small experiment, it appears to be less likely that Sir Francis Bacon is the true author of the works attributed to Shakespeare than Christopher Marlowe although there is little to no evidence that Marlowe was the true author. 

\section{Conclusion}
We have successfully shown that the tree-SNE procedure introduced in \cite{robinson2020tree-sne} generates continuous structures that produce finer and finer clusterings within high-dimensional data. Moreover, by creating tree-SNE embeddings of three high-dimensional datasets, we observed that these clusterings are indeed meaningful and elucidate both macro- and micro-level structures within data. We conclude by describing two interesting directions for further research:
\begin{enumerate}
    \item The tree-SNE procedure can conceivably be adapted to other attraction-repulsion methods to produce trees of other embeddings. We leave the structure of these trees as an interesting direction to explore.
    \item While we have shown that tree-SNE embeddings have continuous strucutre, it is yet unclear why and when different branches form within the tree. To fully understand this, a greater understanding of the inner workings of t-SNE is needed. An interesting question is to uncover these threshold degrees of freedom for highly structured data, such as points sampled from a mixture of mixtures of Gaussians.
\end{enumerate}

\subsection*{Acknowledgements.}
The author is grateful to Stefan Steinerberger for introducing him to this problem and for providing valuable guidance throughout the creation of this manuscript.

\printbibliography
\end{document}